\documentclass[10pt]{article} 

\usepackage[accepted]{rlj} 

\setcitestyle{authoryear,round,citesep={,},aysep={},yysep={-}}

\usepackage{amssymb}            
\usepackage{mathtools}          
\usepackage{mathrsfs}           
\usepackage{graphicx}           
\usepackage{subcaption}         
\usepackage[space]{grffile}     
\usepackage{url}                
\usepackage{lipsum}             

\usepackage{dsfont} 
\usepackage{amsthm}

\newcommand{\cS}{\mathcal{S}}
\newcommand{\cD}{\mathcal{D}}
\newcommand{\cA}{\mathcal{A}}
\newcommand{\RR}{\mathbb{R}}
\newcommand{\EE}{\mathbb{E}}
\newcommand{\PP}{\mathbb{P}}

\theoremstyle{plain}
\newtheorem{theorem}{Theorem}[section]
\newtheorem{proposition}[theorem]{Proposition}
\newtheorem{lemma}[theorem]{Lemma}

\theoremstyle{definition}

\newtheorem{assumption}[theorem]{Assumption}
\theoremstyle{remark}

\definecolor{darkblue}{rgb}{0, 0, 0.55}

\definecolor{mydarkblue}{rgb}{0,0.08,0.45}
\hypersetup{colorlinks,linkcolor={mydarkblue},citecolor={mydarkblue},urlcolor={red}}

\title{Average-DICE: Stationary Distribution Correction by Regression}

\setrunningtitle{Average-DICE: Stationary Distribution Correction by Regression}

\author{Fengdi Che\textsuperscript{1,2}, Bryan Chan\textsuperscript{1,2}, Chen Ma\textsuperscript{1,2}, A. Rupam Mahmood\textsuperscript{1,2,3}}

\emails{{fengdi,bryan.chan,cma2,armahmood}@ualberta.ca}

\affiliations{
$^{1}$\textbf{Department of Computing Science, University of Alberta, Canada}\\
$^{2}$\textbf{Alberta Machine Intelligence Institute (Amii)}\\
$^{3}$\textbf{CIFAR AI Chair}
}

\contribution{
        We reformulate the state distribution ratio between the discounted stationary distribution of the target policy and the undiscounted stationary distribution of the behaviour policy as a new consistent estimator, leveraging a dataset collected under the behaviour policy.
        \item We show that this consistent estimator corrects state distribution shifts in off-policy data, and reweighting each data point with our estimator provides an unbiased estimate for any function.
        \item We introduce Average-DICE, an algorithm that estimates density ratios via regression.
        \item We prove the convergence of our update rules under linear function approximation.
        \item We evaluate Average-DICE against prior algorithms and demonstrate its fast convergence and low policy evaluation error.
    }
    {
    Our algorithm is sensitive to changes in the environment, including the discount factor and the trajectory length.
    }


\summary{
Off-policy policy evaluation (OPE), an essential component of reinforcement learning, has long suffered from stationary state distribution mismatch, undermining both stability and accuracy of OPE estimates.
While existing methods correct distribution shifts by estimating density ratios, they often rely on expensive optimization or backward Bellman-based updates and struggle to outperform simpler baselines.
We introduce Average-DICE, a computationally simple Monte Carlo estimator for the density ratio that averages discounted importance sampling ratios, providing an unbiased and consistent correction.
Average-DICE extends naturally to nonlinear function approximation using regression, which we roughly tune and test on OPE tasks based on Mujoco Gym environments and compare with state-of-the-art density-ratio estimators using their reported hyperparameters.
In our experiments, Average-DICE is at least as accurate as state-of-the-art estimators and sometimes offers orders-of-magnitude improvements. However, a sensitivity analysis shows that best-performing hyperparameters may vary substantially across different discount factors, so a re-tuning is suggested.
}

\begin{document}
\makeCover  

\maketitle  

\begin{abstract}
Off-policy policy evaluation (OPE), an essential component of reinforcement learning, has long suffered from stationary state distribution mismatch, undermining both stability and accuracy of OPE estimates.
While existing methods correct distribution shifts by estimating density ratios, they often rely on expensive optimization or backward Bellman-based updates and struggle to outperform simpler baselines.
We introduce Average-DICE, a computationally simple Monte Carlo estimator for the density ratio that averages discounted importance sampling ratios, providing an unbiased and consistent correction.
Average-DICE extends naturally to nonlinear function approximation using regression, which we roughly tune and test on OPE tasks based on Mujoco Gym environments and compare with state-of-the-art density-ratio estimators using their reported hyperparameters.
In our experiments, Average-DICE is at least as accurate as state-of-the-art estimators and sometimes offers orders-of-magnitude improvements. However, a sensitivity analysis shows that best-performing hyperparameters may vary substantially across different discount factors, so a re-tuning is suggested.

\end{abstract}

\section{Introduction}
Off-policy evaluation (OPE) aims to estimate the expected cumulative return of a target policy using data collected from a different behaviour policy. 
Assessing a policy with pre-collected data before deployment is crucial, as executing an unqualified policy can lead to undesirable consequences \citep{levine2020offline}, including life-threatening risks in applications such as surgical robotics and self-driving vehicles. 
A straightforward approach is to directly average the observed rewards.
However, distribution shift introduces bias into value estimation, and even temporal difference methods cannot provide an unbiased evaluation of the target policy under such shifts \citep{sutton2016emphatic}.

A common approach to addressing distribution shift is importance sampling (IS) \citep{precup2001off}, which reweights samples based on the ratio between two distributions to provide an unbiased estimation. 
However, when correcting cumulative returns along a trajectory, IS requires multiplying these ratios over multiple steps, leading to high variance — a problem known as the curse of the horizon. 
To mitigate this, researchers have explored marginalized IS ratios for stationary state distributions \citep{liu2018breaking}. 
Current estimators leverage the recursive property of stationary distributions to formulate optimization tasks. 
This recursion results in a backward Bellman-based update, where the value at the next step depends on the current step’s value \citep{hallak2017consistent}. 
However, the expectation in the backward Bellman recursion cannot be unbiasedly evaluated without double sampling. Moreover, off-policy Bellman updates with function approximation, known as the deadly triad \citep{baird1995residual}, are prone to instability and typically lack convergence guarantees.

Several studies have proposed novel optimization frameworks, such as primal-dual optimization \citep{nachum2019dualdice} or multistage optimization \citep{uehara2020minimax}, to avoid directly minimizing the backward Bellman error. 
The optimal solutions corresponding to their novel losses equal to the desired state density ratio, proven in the tabular case \citep{liu2019off}. 
However, these methods introduce additional complexity to the estimation process.
This raises an important question: Can we revisit Monte Carlo methods to develop a new estimator that is both theoretically sound and computationally simpler? 
Previously, a state distribution corrector based on the Monte Carlo expansion of the stationary distribution was developed to account for the missing discount factor in stationary state distributions for policy gradient algorithms \citep{che2023correcting}. However, the idea was not explored in the off-policy setting.

In this paper, we propose a novel estimator for the stationary state distribution ratio,
called the average state distribution correction estimation (Average-DICE).
It leverages the Monte Carlo expansion rather than the recursive property used in prior approaches.
Our approach computes the average of all discounted importance sampling ratio products corresponding to a given state in the dataset. 
We prove that this method provides a consistent estimation of the density ratio between the discounted target and the undiscounted behaviour stationary state distributions. 
Also, it gives an unbiased estimation of any function.

Furthermore, our estimator can be learned via a least squares regression task, offering a simple and effective approach to approximating the state distribution ratio. 
In the case of linear function approximation, we establish its asymptotic convergence to the same fixed point as minimizing the mean squared error with the exact density ratio, under standard assumptions used in temporal difference (TD) convergence analysis \citep{yu2015convergence}.

To evaluate our estimator, we conduct experiments on several discrete classic control tasks and continuous MuJoCo tasks \citep{todorov2012mujoco}, using a pre-collected fixed off-policy dataset with batch updates. 
Our estimator achieves dominant performance on most tasks when appropriately tuned for the required discount factor and remains competitive on others. 
Additionally, it demonstrates the fastest convergence to a stable value, making it practical for integration into other algorithms. 
However, our algorithm is sensitive to changes in the discount factor, and we recommend re-tuning it for each discount setting to ensure optimal performance.

\section{Background}
\paragraph{Notation} 
We let $\Delta(\mathcal{X})$ denote the set of probability distributions over a finite set $\mathcal{X}$. 
Let $\RR$ denote the set of real numbers, $\mathbb{N}$ be the set of non-negative integers, $\mathbb{N}^+$ be the set of positive integers, and $\mathds{1}$ be the indicator function.

\paragraph{Markov Decision Process}
We consider finite Markov decision process (MDP) \citep{sutton2018reinforcement} defined by a tuple $M=\langle \mathcal{S} \cup \{\varsigma\},\mathcal{A},r, P,\nu,\gamma\rangle$, where 
$\cS$ is a finite  state space,
$\varsigma$ is a termination state, 
$\cA$ is the action space, 
$r:\cS\times\cA\rightarrow \RR$ is the reward function,
$P:\cS\times\cA\rightarrow \Delta(\cS)$ is the transition matrix, 
$\nu\in \Delta(\cS)$ is the distribution of the initial state, and
$\gamma \in (0,1)$ is the discount factor.   
At each step $j$, the agent applies an action $A_j$ sampling from a policy $\pi: \cS \rightarrow \Delta(\cA)$ at state $S_j$.
Then, the agent receives a reward $R_j$ and transits to the next state $S_j'$. Our paper focuses on episodic tasks, where a trajectory,
denoted by $\tau = \{S_j,A_j,R_j,S_j'\}_{j=0}^{T-1}$, 
ends at the termination state $\varsigma$ at step $T$. 
We define the random variable $T \in \mathbb{N}^+$ as \emph{the length of the trajectory}. 

The agent's goal is to evaluate a policy $\pi$ by estimating the expected discounted cumulative returns. 
The expected discounted cumulative return, denoted by $J(\pi)$, is defined as 
\begin{equation*}
\abovedisplayskip=0.2pt
J(\pi) = (1-\gamma)\EE_\pi \left[ \sum_{j=0}^{\infty} \gamma^j r(S_j,A_j) \right],
\end{equation*}
where we use $\EE_\pi$ to denote the expectation under the distribution induced by 
$\pi$ and the environment.  
Notice that after reaching the termination state, the probability of any state showing up is zero, that is, $\PP_{\pi}(S_j =s) = 0\, , \forall s \in \cS$, if $T \le j$ and rewards also equal zero.

The \emph{Q-value} represents the expected cumulative rewards starting from a state-action pair $(s,a)$ following a policy $\pi$, defined as
\begin{equation}
\abovedisplayskip=0.2pt
q_\pi(s, a) = \EE_\pi\left[ \sum_{j=0}^{\infty} \gamma^j r(S_j,A_j) \bigg| S_0=s, A_0=a\right]\, .
\belowdisplayskip=-2pt
\end{equation}

\paragraph{Off-Policy Evaluation}
We consider evaluating a target policy $\pi$ using a dataset $\cD$ that consists of $K$ trajectories as 
$\{\tau_i\}_{i=1}^K = \left \{ \left \{S_j^i,A_j^i,R_j^i,(S_j')^i ) \right\}_{j=0}^{T^i-1} \right \}_{i=1}^K$.
When trajectories are collected under a \emph{behaviour policy},
denoted by $\mu$,
differing from the target policy, we call the learning off-policy.

To correct this distribution shift, the \emph{importance sampling (IS) ratio}
$\rho(a|s) = \frac{\pi(a|s)}{\mu(a|s)}$ is often used to obtain unbiased estimators. 
The \emph{IS-ratio product}, 
denoted by $\rho_{0:j-1} = \prod_{k=0}^{j-1}\rho(A_k|S_k)$, 
adjusts the distribution of an entire trajectory, $\{S_0,A_0,\cdots,S_j\}$, from the behaviour policy to the target one. 
Notice the product is initialized at one, denoted as
$\rho_{0:-1} =1$,
with some abuse of notations.

This dataset can also be expressed in terms of individual transitions as 
$\mathcal{D}=\{(S_t,A_t,R_t,S_t',\textrm{time}_t,\rho_{\textrm{prod},t})\}_{t=0}^{n-1}$ where
$n$ is the dataset size,
$t$ is the index for each transition,
$\textrm{time}_t$ represents the step of $s_t$ in its trajectory and $\rho_{\textrm{prod},t}$ for the corresponding IS products $\rho_{0:\textrm{time}_t-1}$ until $s_t$. 
To further simplify notation, let $I_s$ indicate the set of step $t$ such that $S_t=s$.

Off-policy TD estimates Q-values by $\hat{q}_{\theta}$ and takes a semi-gradient of the empirical temporal difference errors, which equals
\begin{equation}
\abovedisplayskip=2pt
    \min \mathcal{L}(\theta;\cD) = \frac{1}{n} \sum_{t=0}^{n -1 } \left ( R_t + \gamma \hat{q}_{\theta}\left (S_t',A_t' \right ) - \hat{q}_{\theta}(S_t,A_t) \right )^2,
    \belowdisplayskip=2pt
\end{equation}
where the next action used for the bootstrapping target is sampled from the target policy, that is,
$A_t' \sim \pi(\cdot|S_t')$ .
TD evaluates the target policy by expected value estimation of initial state-action pairs, that is, 
$(1-\gamma) \EE_{S_0 \sim \nu, A_0 \sim \pi(\cdot|S_0)} \left [ \hat{q}_{\theta} (S_0,A_0)\right ] $. However, the dataset’s state distribution shift is not corrected from the behaviour policy to the target policy, leading to bias in the estimation.

\paragraph{Irreducible Markov Chain}
The Markov decision process under a policy $\pi$ forms a Markov chain, denoted by $\langle \mathcal{S} \cup \{\varsigma\}, P_{\pi}\rangle$,
where $P_{\pi}(s'|s) = \sum_{a\in\mathcal{A}}\pi(a|s)P(s'|s,a)$ for any $s, s' \in \cS$. 
A Markov chain is said to be \emph{irreducible}, 
if for any two states, $s$ and $s'$, the probability of transiting between these two states is positive at some time step, 
that is, $P_{\pi}(S_j =s' \text{ for some } j>0 |S_0=s)>0$ and $P_{\pi}(S_j =s \text{ for some } j>0 |S_0=s')>0$.

The \emph{recurrence time} of a state $s$, denoted by $\tau_s^+(s)$,
is defined as the time elapsed to revisit a state $s$, that is, $\tau^+_s(s) = \min \{j>0:S_j =s ,S_0 =s\}$.
A \emph{positive recurrent} state has a finite expected recurrence time, that is, $\mathbb{E}_\pi[\tau^+_s(s)] < \infty$. Note that no assumptions are made in the background section; the terms irreducibility and positive recurrence are presented solely for later use.

\paragraph{Stationary State Distribution}
The discounted stationary state distribution, denoted by $d_{\pi,\gamma}$, is defined as the distribution satisfying the following equation for all states $s' \in \mathcal{S}$:
\begin{equation}
\sum_{s\in\mathcal{S}} d_{\pi,\gamma}(s)[\gamma P_{\pi}(s'|s)+(1-\gamma) \nu(s')] = d_{\pi,\gamma}(s') .
\label{eq:backward}
\belowdisplayskip=0.2pt
\end{equation} 
A common analytical form of the discounted stationary distribution can be written as
\begin{equation}
\abovedisplayskip=0.2pt
d_{\pi,\gamma}(s) = (1-\gamma)\sum_{j = 0}^\infty \gamma^j \PP_{\pi}(S_j=s).
\label{eq:def_stationary}
\belowdisplayskip=0.2pt
\end{equation}

Given the form of the discounted stationary distribution, the expected discounted cumulative return, $J(\pi)$, can also be written as 
\begin{equation}
\abovedisplayskip=0pt
    J(\pi) = \sum_{s\in\cS} d_{\pi,\gamma}(s) r_{\pi}(s),
    \label{eq:form_obj}
\end{equation}
where $r_{\pi}(s) = \EE_{A \sim \pi(\cdot|s)}[r(s,A)]$.

The undiscounted stationary distribution, denoted by $d_{\pi}$ is defined as the distribution satisfying $\sum_{s\in\mathcal{S}}d_{\pi}(s)P_{\pi}(s'|s)=d_{\pi}(s')$. 
This distribution is also regarded as the limiting distribution of state-action visitation at each step. However, in episodic tasks, the step count does not approach infinity. To define a limit distribution in this setting, the trajectory is considered to restart from the initial state distribution upon termination. Note that this restart does not occur in practice and is introduced purely for definitional purposes.

Repeating the transition $n$ steps can give $K$ terminated trajectory and one incomplete trajectory, denoted by  $\left \{ \left \{\left(S_j^i,A_j^i,R_j^i,(S_j')^i \right) \right\}_{j=0}^{T^i-1} \right \}_{i=1}^K \cup \left\{\left(S_j^{K+1},A_j^{K+1},R_j^{K+1},(S_j')^{K+1} \right) \right\}_{j=0}^{n - \sum_{k=1}^K T^k} $. We relabel the transition by $t$ as $\left \{(S_t,A_t,R_t,S_t')\right \}_{t=0}^{n-1} $.
The undiscounted stationary distribution has multiple analytical forms stated in Sutton and Barto \citeyearpar{sutton2018reinforcement} and  Grimmett and Stirzaker (\citeyear{grimmett2020probability}, Theorem 6.4.3), summarized in a lemma from \citep{che2023correcting}.
\begin{lemma}[Forms of Undiscounted Stationary Distribution] \label{lemma:form_stat_dist}
Under the irreducibility of the Markov chain and positive recurrences of all states under all policies $\pi$, we have the following:
  \begin{equation}
  \abovedisplayskip=0pt
      d_{\pi}(s) = \lim_{n \to \infty} \frac{1}{n} \sum_{t=0}^{n-1} \PP_{\pi}(S_t=s)  = \frac{1}{\EE_\pi[\tau^+_s(s)]}.
  \end{equation}
\end{lemma}

\section{Related Works}
TD with linear function approximation converges when data is sampled as trajectories under the target policy \citep{tsitsiklis1996analysis}, but linear TD with off-policy state distribution is not guaranteed to converge \citep{che2024target}. This issue is called the deadly triad. Meanwhile, the policy estimation is biased under the state distribution shift.

The data distribution can be corrected by importance sampling \citep{precup2000eligibility,precup2001off}. However, these approaches suffer from high variance when correcting the distributions of trajectories with products of IS ratios. Later papers work on estimating state distribution ratios to avoid the ratio product \citep{hallak2017consistent,yang2020off,fujimoto2021deep}.

The state distribution ratio can be estimated based on the backward recursion for the stationary distribution shown in Equation \ref{eq:backward}. A backward Bellman recursion for the density ratio $w(s)$ can then be built for all state $s'$, and the temporal difference error for the density ratio estimator, denoted by $TD(s')$, is defined as
\begin{equation*}
\abovedisplayskip=0.2pt
    TD(s'):=\mathbb{E}_{(S,A,S') \sim d_{\mu}} \left [ -w(S') + \gamma w(S) \frac{\pi(A|S)}{\mu(A|S)} | S'=s'\right ] + (1-\gamma) \rho(s').
\end{equation*}
This temporal error equals zero, if $w(s)=\frac{d_{\pi,\gamma}(s)}{d_{\mu}(s)}$ provided that non-zero target policy $\pi(a|s)>0$ implying the behaviour policy being non-zero $\mu(a|s)>0$ for all state-action pairs \citep{nachum2019dualdice}.
COP-TD \citep{hallak2017consistent,gelada2019off} minimizes the above temporal difference (TD) error. However, a backward TD estimate cannot be unbiasedly computed from a dataset without double sampling unless the behaviour policy is concentrated on a single state. Meanwhile, the algorithm lacks a convergence guarantee. 

Several other works \citep{liu2018breaking,liu2019off,uehara2020minimax} design novel loss functions based on the recursive properties of the state distribution instead of directly minimizing the TD error. These losses reach zero if and only if the solution is the density ratio, providing new multi-stage optimization objectives for ratio approximation. On the other hand, DualDice \citep{nachum2019dualdice} introduces a primal-dual optimization framework by reformulating the problem with the Fenchel conjugate. GenDice \citep{zhang2020gendice} estimates the density ratio $w(s)$ and minimizes the f-divergence between the estimated and true stationary distributions $d_{\mu}(s)w(s)$ and $d_{\pi,\gamma}(s)$, showing greater stability than DualDice for high discount factors but lacking convex-concavity. However, these multi-stage or primal-dual optimization techniques lack the convergence guarantee, and the training is less stable with multiple variables.

GradientDice \citep{zhang2020gradientdice} replaces f-divergence in GenDice with a weighted L$2$-norm, ensuring convex-concave and convergence properties under linear function approximation. 
BestDice \citep{yang2020off} unifies these multi-stage and primal-dual methods into a general objective, identifying optimal regularization choices in their BestDice algorithm. However, the learning stability still needs improvement.

Successor Representation Distribution Correction Estimation (SR-DICE) \citep{fujimoto2021deep} builds on successor features and derives a loss equivalent to minimizing the mean squared error to the density ratio under linear function approximation. It achieves lower policy evaluation errors than other density estimators but still underperforms deep off-policy TD. Meanwhile, state-action representation features and successor features require pre-training, introducing additional approximation errors and increasing computation.

\section{Distribution Corrector}
We derive a novel expression for the state density ratio, leading to a consistent estimator. 
This estimator computes the average of discounted IS-ratio products for each state using an off-policy dataset. 
Our algorithm, Average-DICE, is named for its averaging approach in approximating this estimator.
As the dataset size approaches infinity, our estimator converges to the true density ratio. 
Meanwhile, it corrects the distribution shift from the dataset’s sampling distribution to the target policy’s discounted stationary distribution, consequently providing an unbiased estimate for any function by reweighting each state by our estimator.

Recall that a dataset consists of $K$ trajectories and is presented as
$\mathcal{D}=\{(S_t,A_t,R_t,S_t',\textrm{time}_t,\rho_{\textrm{prod},t})\}_{t=0}^{n-1}$, 
where $\textrm{time}_t$ represents the step of $S_t$ in its trajectory and $\rho_{\textrm{prod},t}$ for the corresponding IS products until $S_t$.
$I_s$ indicates the set of label $t$ such that $S_t=s$.

We first assume the following necessary condition for applying marginalized importance sampling.

\begin{assumption}
    If $d_{\pi,\gamma} > 0$, then $d_{\mu} > 0$.
    \label{assump:stationary_abso_cont}
\end{assumption}
This assumption is made for all distribution correction estimators \citep{yu2015convergence,zhang2020gradientdice}, requiring the off-policy distribution to cover the target distribution.

Also, for episodic tasks, it is normal to consider the trajectory length to have a finite expectation.
\begin{assumption}
    $$\EE_{\mu} [T] < \infty.$$
    \label{assump:bound_traj_len}
\end{assumption}
\vspace*{-4mm}

Now we are ready to present the novel formulation of the density ratio.
\begin{proposition}
\label{thm:density-ratio}[Consistency]
Given
\begin{itemize}
    \item a finite Markov decision process,
    \item a dataset $\cD$ collected under a behaviour policy $\mu$, and
    \item a target policy $\pi$
\end{itemize} 
such that Assumption \ref{assump:stationary_abso_cont} and \ref{assump:bound_traj_len} are satisfied, 
then for state $s$ with $$d_{\pi,\gamma}(s) >0,$$
we have the density ratio equal
\begin{equation}
    \frac{d_{\pi,\gamma}(s)}{ d_{\mu}(s)} = \lim_{n \to \infty}  \frac{n}{K}(1-\gamma)\mathbb{E}_{t \sim I_s} [\gamma^{\textrm{time}_t}  \rho_{\textrm{prod},t}], \label{eq:density-ratio}
\end{equation}
where $n$ is the number of transitions, and 
$K$ denotes the number of trajectories.
\end{proposition}

We derive a consistent estimator based on the above proposition, defined as
\begin{equation}
    \label{eq:derived_estimator}
    c_{\cD}(s) =  \frac{n}{K}(1-\gamma)\mathbb{E}_{t \sim I_s} [\gamma^{\textrm{time}_t}  \rho_{\textrm{prod},t}].
\end{equation}
The expectation is taken over step where state $s$ appears and can be expressed as
\begin{equation}
   \mathbb{E}_{t \sim I_s} [\gamma^{\textrm{time}_t}  \rho_{\textrm{prod},t}] = \frac{\sum_{t=0}^{n-1} \gamma^{\textrm{time}_t}  \rho_{\textrm{prod},t} \mathds{1}[S_t=s ]}{\sum_{t=0}^{n-1}  \mathds{1}[S_t=s ]},
\end{equation}
where $\mathds{1}[S_t=s ]$ equals to one if $s$ appears at step $t$. The denominator counts the number of times state $s$ occurs in the dataset, while the numerator sums the corresponding discounted IS-ratio products. Thus, the expectation in our estimator effectively averages all discounted IS products associated with state $s$.

Our main theorem shows that reweighting each data by our estimator gives an unbiased estimator for any function.

\begin{theorem}[Unbiasedness]
\label{thm:unbiasedness}
Given
\begin{itemize}
    \item a finite Markov decision process,
    \item a dataset $\cD$ collected under a behaviour policy $\mu$, and
    \item a target policy $\pi$
\end{itemize} 
such that Assumption \ref{assump:stationary_abso_cont} and \ref{assump:bound_traj_len} are satisfied, 
reweighting data by our average correction gives unbiased estimation for any function $f:\cS \to \RR$, that is,
\begin{equation}
\EE_{\cD} \left [\EE_{S \sim \cD}\left [c_{\cD}(S) f(S)\right ] \right ]  = \EE_{S \sim d_{\pi,\gamma}}[f(S)],
\end{equation}
where $\EE_{\cD}$ means expectation over trajectories sampled under the behaviour policy, and
$\EE_{S \sim \cD}$ representing sampling states uniformly from the dataset.
\end{theorem}

This theorem holds because our estimator equals the ratio of an unbiased and consistent estimation of the discounted target distribution, denoted as $\hat{d}_{\pi,\gamma}(s)=\frac{1}{K}\sum_{i=1}^K\sum_{j \ge 0} \gamma^j \rho_{0:j-1}^i \mathds{1}[S_j^i=s]$ to the sampling distribution from the dataset, denoted as $\hat{d}(s) = \cfrac{\sum_{t=0}^{n-1}  \mathds{1}[S_t=s ]}{n}$. Therefore, as long as we can calculate our estimator, the distribution shift can be solved.

\section{Average-DICE Algorithm}
In this section, we work on how to evaluate our derived estimator $c_{\cD}(s)$.
Our estimator averages the corresponding discounted IS-ratio products for a state $s$. 
However, in high-dimensional state spaces, direct averaging by state counting is infeasible.
Thus, we propose to approximate the expectation of discounted IS-ratio products via regression.

We first introduce our regression losses and propose the Average-DICE algorithm. Then, we show that with linear function approximation, incrementally updating our loss results in a convergent algorithm. The fixed point of this update corresponds to minimizing the mean squared error (MSE) to the true density ratio with regularization.

\subsection{Loss}
We learn our estimator as a ratio model by minimizing the least squares error.
Solving least squares regression with Markovian data is well studied but generally requires more samples compared to i.i.d. learning tasks \citep{nagaraj2020least}.
In our setup, the ratio model takes states as inputs and 
is trained by minimizing the mean squared error between its output  $f_{\theta}(s_t)$ and its corresponding regression target $\gamma^{\textrm{time}_t}  \rho_{\textrm{prod},t}$. 
In this case, the ratio is estimated by $\frac{n}{K}(1-\gamma) f_{\theta}(s)$.
The expected discounted cumulative return can be estimated by 
\begin{equation}
    \hat{J}(\pi) = \frac{1}{n} \sum_{t=0}^{n-1} \frac{n}{K}(1-\gamma) f_{\theta}(S_t) R_t.
\end{equation}

In the regression with a fixed dataset, a parameter regularization is usually added to avoid overfitting. Our algorithm also uses $\frac{\lambda_1 \lVert \theta \rVert_2^2}{2}$ to regularize, where $\lambda_1$ is the regularization parameter.

Meanwhile, same as GradientDice, the learnt ratio should ensure that $\sum_s d_{\mu}(s) \frac{n}{K}(1-\gamma) f_{\theta}(s) \approx \sum_{s}d_{\mu}(s)\frac{d_{\pi,\gamma(s)}}{d_\mu(s)} =1$. So our algorithm further regularizes by the loss $\frac{\lambda_2}{2} (\sum_s d_{\mu}(s) \frac{n}{K} (1-\gamma) f_{\theta}(s) -1)^2$, called the \emph{distribution regularization}.
An expectation in a square loss cannot be estimated unbiasedly using samples. Thus, this regularization term is re-written by the Fenchel conjugate as
\begin{equation}
    \lambda_2 (\max_{\eta \in \RR} \EE_{s \sim d_{\mu}}[\eta  \frac{n}{K}(1-\gamma) f_{\theta}(s) - \eta] -\frac{\eta^2}{2}).
\end{equation}

The loss given a dataset $\cD$ is written as 
\begin{align}
    \min_{\theta}\mathcal{L}(\theta; \cD) := \EE_{S_t \sim \cD}&\left [ \frac{1}{2}(f_{\theta}(S_t) - \gamma^\textrm{time}_t\rho_{\textrm{prod},t} )^2 \right ] + \frac{\lambda_1 \lVert \theta \rVert_2^2}{2} \nonumber\\
    &+ \lambda_2 \left (\max_{\eta \in \RR} \EE_{S_t \sim \cD}[\eta  \frac{n}{K}(1-\gamma) f_{\theta}(S_t) - \eta] -\frac{\eta^2}{2} \right ). 
    \label{eq:loss}
\end{align}

\subsection{Convergence Analysis}
This section focuses on the linear function approximation with $f_{\theta}(s) = \phi(s)^\top\theta$, where $\phi(s)\in\RR^d$ is a given state feature and $\theta\in\RR^d$ is the parameter. We denote $\Phi\in\RR^{|\cS|\times d}$ as the feature matrix, where each row corresponds to the feature vector of a particular state $s$.

At each step $t$, the agent takes an action according to the behaviour policy at state $s_t$. If the trajectory terminates, the agent restarts according to the initial distribution. The algorithm updates the parameters $\theta$ and $\eta$ in our distribution regularization at each step with the newly collected transition following our loss shown in Equation \ref{eq:loss}. The regression target, denoted by $y_t = \gamma^\textrm{time}_t\rho_{\textrm{prod},t}$, equals to the discounted IS-ratio products computing using the state's current trajectory, where $\textrm{time}_t$ represents the step of the state in its current trajectory and $\rho_{\textrm{prod},t} = \rho_{0:\textrm{time}_t-1}$. 

Instead of using a running scalar of $\frac{t}{K}$ in the loss, we evaluate an average trajectory length $H$ at the beginning and keep it fixed. This fixed multiplier simplifies the proof. We hypothesize that using the original one $\frac{t}{K}$ converges as well but with high probability instead of almost surely, since $\frac{t}{K}$ may not be bounded for all $t \in \mathbb{N}$. However, both two scalars are estimating the average trajectory length $\EE_{\mu}[T]$ and are close.
\setlength{\abovedisplayskip}{3.5pt}
\setlength{\belowdisplayskip}{3.5pt}

The update rule is
\begin{align}
    \eta_{t+1} &= \eta_t + \alpha_t \lambda_2 (H(1-\gamma)\phi(s_t)^\top\theta_t - 1- \eta_t).\\
    \theta_{t+1} &= \theta_t - \alpha_t (\phi(s_t) (\phi(s_t)^\top\theta_t - y_t)+\lambda_2 \eta_t H(1-\gamma) \phi(s_t) + \lambda_1 \theta_t),
\end{align}
where $\alpha_t$ is the learning rate. We combine the system of equations into $$d_{t+1} = d_t + \alpha_t (G_{t+1} d_t + g_{t+1}),$$ 
where $d_{t+1} = 
\begin{bmatrix}
    \theta_{t+1} \\
    \eta_{t+1}
\end{bmatrix}$ denotes the concatenation of parameters, 
 and
update matrices are $$G_{t+1} = 
\begin{bmatrix}
    -\phi(s_t) \phi(s_t)^\top - \lambda_1 I & -\lambda_2 H (1-\gamma) \phi(s_t)\\
    \lambda_2 H(1-\gamma) \phi(s_t) & -\lambda_2
\end{bmatrix} \textrm{and } \;
g_{t+1} = 
\begin{bmatrix}
    \phi(s_t) y_t \\
    -\lambda_2
\end{bmatrix}.$$

\begin{assumption}
\begin{enumerate}
    \item $\Phi$ has linearly independent columns.
    \item Each feature vector $\phi(s)$ has its L$2$-norm bounded by $L$.
    \item The behaviour policy $\mu$ induces an irreducible Markov chain on $\cS$ and moreover, for all $(s,a) \in \cS \times \cA$, $\mu(a|s)>0$ if $\pi(a|s)>0$.
    \item The stepsize sequence $\{\alpha_t\}$ is deterministic and eventaully nonincreasing, and satisfies $\alpha_t \in (0,1]$, $\sum_t \alpha_t = \infty$, and $\sum_t \alpha_t^2 < \infty$.
\end{enumerate}

\label{assump:convergence}
\end{assumption}

These four assumptions are also assumed for ETD convergence analysis and are common for analyzing the asymptotic behaviours of linear update rules.

Define two matrices  $$G = 
\begin{bmatrix}
    -\Phi^\top D_{\mu} \Phi - \lambda_1 I  & -\lambda_2 H (1-\gamma)\Phi^\top d_{\mu}\\
    \lambda_2 H (1-\gamma) d_{\mu}^\top \Phi & -\lambda_2
\end{bmatrix}, \textrm{ and }  g = 
\begin{bmatrix}
    \frac{1}{(1-\gamma)\EE_{\mu}[T]} \Phi^\top D_{\mu} y\\
    -\lambda_2
\end{bmatrix},$$\\
where $y \in \RR^\cS$ denotes the density ratio $\frac{d_{\pi,\gamma}(s)}{d_{\mu}(s)}$.


Incremental updates under our losses give convergence with linear function approximation. The proof follows the convergence analysis of ETD and is presented in Appendix B. Intuitively, our correction gives a consistent estimator with variance controlled by the discount factor, and thus, the convergence follows.

\begin{theorem}
Based on Assumption \ref{assump:bound_traj_len} and \ref{assump:convergence}, we have
    \begin{equation}
        d_t \to -G^{-1} g \text{ a.s.}
    \end{equation}
which gives the same fixed point for minimizing the mean square error to the true density ratio, which is $\EE_{S_t \sim \cD} \left [ \frac{1}{2} \left (f_{\theta_{\textrm{mse}}}(S_t) - \frac{1}{(1-\gamma)\EE_{\mu}[T]} \frac{d_{\pi,\gamma}(S_t)}{d_{\mu}(S_t)} \right )^2 \right ]$ with the same regularizations.
\end{theorem}

\begin{figure}[t]
    \centering
    \vspace*{-6mm}
    \hspace*{-2mm}
    \includegraphics[width=1.1\textwidth
    ]{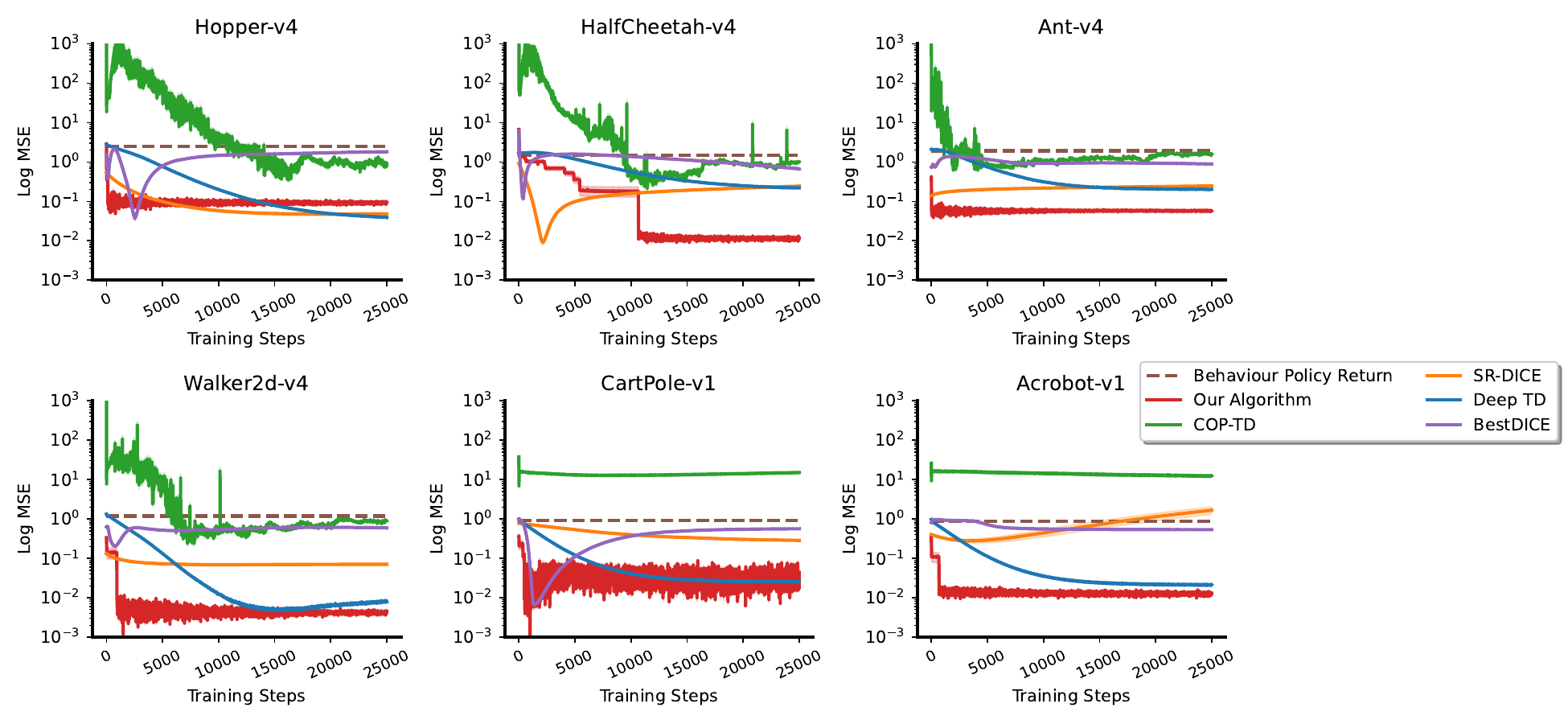}
    \vspace*{-2mm}
    \caption{This figure presents the mean square error of estimating the objective $J(\pi)$ in the log scale for each task. Our method, as the red line, shows dominant behaviour on most tasks and comparable behaviour on Hopper and CartPole. }
    \label{fig:tune}
    \vspace*{-6mm}
\end{figure}

\section{Experiments}
We perform OPE on classic control and MuJoCo \citep{todorov2012mujoco} tasks to evaluate our method and compare it with other distribution correctors, including COP-TD, BestDice, and SR-DICE. Additionally, we include two simple baselines: the average reward, which represents the objective under the behaviour policy, and off-policy TD.
In these OPE tasks, the target policy is trained using PPO \citep{schulman2017proximal}, which can achieve high-performing policies; for example, the agent for CartPole receives above $410$ return, close to the optimal return of $500$. For discrete actions, the behaviour policy is a combination of the target policy and the uniform random policy. For continuous actions, the behaviour policy is obtained by increasing the variance of the Gaussian target policy.

The hyperparameters are tuned using a dataset with $4000$ transitions coming from trajectories each of length $100$. The discount factor is fixed at $0.95$, which is a common choice. The random policy weights $0.3$ in the behaviour policy for discrete-action tasks, and the variance is doubled for continuous-action tasks. We selected the combination of hyperparameters that yields the lowest objective estimation error, averaged across all tasks. The results are averaged among $10$ seeds, and the variance is tiny due to similar rewards received per step for each run.

Our results in Figure \ref{fig:tune} show that most of the existing methods underperform compared to off-policy TD, confirming prior work \citep{fujimoto2021deep}.
Only SR-DICE, in the orange line, can give comparable behaviour. Our method, as the red line, shows dominant behaviour on most tasks and comparable behaviour on Hopper and CartPole. More surprisingly, it gives the fastest convergence to a stable low error. Also, in Figure \ref{fig:tune}, our method is tuned specifically to the trajectory length, the randomness of the behaviour policy, the size of the dataset, and the discount factor, which gives a small advantage to our algorithm. Thus, in the next step, we test out the robustness against more settings, as illustrated in Figure \ref{fig:setting_walker}. This figure presents the results of the Walker task, while results for other tasks are provided in Appendix C.

The top-left subfigure of Figure \ref{fig:setting_walker} examines robustness against different discount factors. Our algorithm proves less robust to changes in the discount factor and loses its leading performance, yielding worse results than TD and SR-DICE. Because altering the discount factor significantly changes the regression targets for the entire dataset, it is reasonable that our method would require re-tuning for each discount factor to achieve optimal performance.

When the discount factor is fixed at $0.95$, our method generally maintains a dominant or at least comparable performance relative to other baselines, except in Acrobot and Hopper. In Acrobot, it still outperforms other density-ratio estimators in most settings and is on par with TD; in Hopper, both our method and SR-DICE perform similarly to TD. Also, an ablation study without the distribution regularization is given in Appendix C.

\begin{figure}[ht]
    \centering
    \vspace*{-4mm}
    \hspace*{-2mm}
    \includegraphics[width=0.75\textwidth
    ]{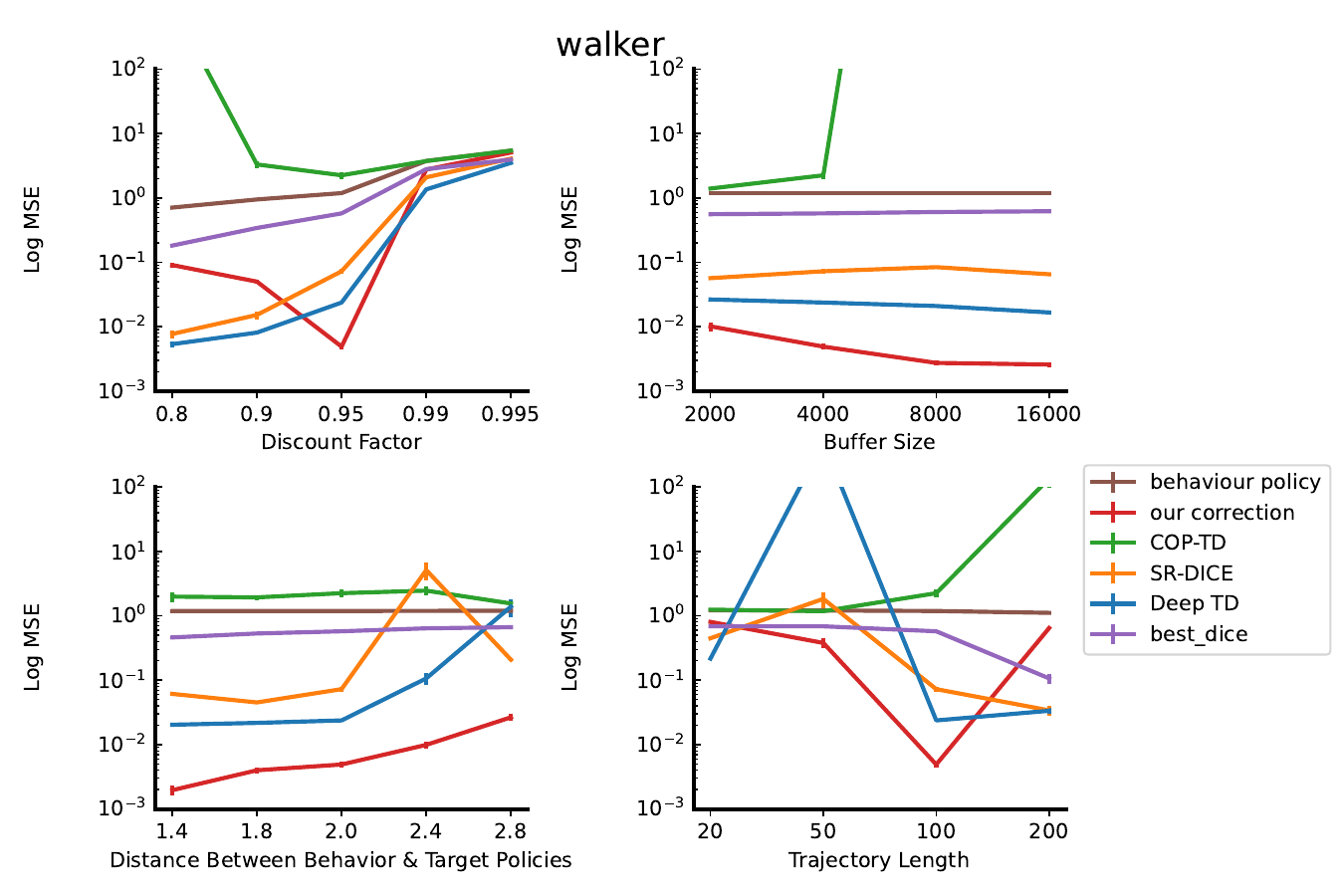}
    \vspace*{-2mm}
    \caption{We evaluate the robustness of our method under varying dataset sizes, trajectory lengths, behaviour policies, and discount factors. This figure presents the results of the Walker task. Our algorithm is less robust only to changes in the discount factor and excessively long trajectory length.}
    \label{fig:setting_walker}
    \vspace*{-6mm}
\end{figure}

\section{Conclusion}
We introduced Average-DICE, a novel regression-based estimator for the stationary state density ratio. Our key contributions include deriving an alternative form of the state density ratio, proposing a novel distribution corrector and designing the learning algorithm for our distribution corrector. Furthermore, we showed that incremental updates converge under linear function approximation, demonstrating that the resulting fixed point coincides with the minimum MSE solution to the true ratio up to regularization. Empirical results on discrete and continuous tasks confirmed that Average-DICE provides stable and accurate off-policy evaluation. Looking forward, integrating this density ratio correction into policy gradient algorithms could address distribution mismatches more effectively than current conservative policy updates. 
\bibliography{main}
\bibliographystyle{rlj}

\beginSupplementaryMaterials

\section*{Appendix A: Derivation of Our Distribution Corrector}
\label{sec:appendixA}

\subsection*{Proof of the Consistency Theorem}
We fist present a theorem used for the proof.

\begin{theorem}[Theorem 1.0.2 Ergodic theorem, Norris (1998)]
Let $\mathcal{M}$ be an irreducible and positive recurrent Markov decision process for all policies. Then, for each state $s$,
\begin{align*}
    \mathbb{P}\left(\frac{\sum_{t=0}^{T-1} \mathds{1}[S_t=s]}{T} \to \frac{1}{\mathbb{E}[\tau_s^+(s)]} \text{ as } T \to \infty\right) =1.
\end{align*}
\end{theorem}

Next, we present our main theorem and show that our correction term equals the distribution ratio, $\frac{d_{\pi,\gamma}(s)}{d_{\mu}(s)}$ and thus, this term successfully corrects the state distribution shift.

\begin{theorem}[Consistency]
Given
\begin{itemize}
    \item a finite Markov decision process,
    \item a dataset $\cD$ collected under a behaviour policy $\mu$, and
    \item a target policy $\pi$
\end{itemize} 
such that Assumption \ref{assump:stationary_abso_cont} is satisfied, 
then for state $s$ with $$d_{\pi,\gamma}(s) >0,$$
we have the density ratio equal
\begin{equation}
\frac{d_{\pi,\gamma}(s)}{ d_{\mu}(s)} = \lim_{n \to \infty}  \frac{n}{K}(1-\gamma)\mathbb{E}_{t \sim I_s} [\gamma^{\textrm{time}_t}  \rho_{\textrm{prod},t}] . \label{eq:density-ratio}
\end{equation}
\end{theorem}

\begin{proof}

\textbf{Reformulate the RHS.}

When sampling $t \sim I_s$ uniformly, the probability equals $\cfrac{\mathds{1}[S_t=s]}{\sum_{k=1}^n \mathds{1}[S_k=s]}$. Furthermore,
\begin{align}
\frac{n}{K}(1-\gamma)\mathbb{E}_{t \sim I_s} [\gamma^{\textrm{time}_t}  \rho_{\textrm{prod},t}]  &=  \frac{n}{K}(1-\gamma) \sum_{t=1}^n\cfrac{\mathds{1}[S_t=s] \gamma^{\textrm{time}_t}  \rho_{\textrm{prod},t} }{\sum_{k=1}^n \mathds{1}[S_k=s]}  \\
&=  (1-\gamma) \cfrac{n}{\sum_{t=0}^{n-1} \mathds{1}[S_t=s]} \frac{1}{K} \sum_{i=1}^K\sum_{j\ge0}\mathds{1}[S^i_j=s] \gamma^{j}  \rho^i_{0:j-1}  .
\end{align}
Note that $n$ is the number of transitions, and $K$ is the number of trajectories.

Define two functions:
\begin{align}
    g_n(s) &= \cfrac{n}{\sum_{t=0}^{n-1} \mathds{1}[S_t=s]}.\\
    f_n(s) &= (1-\gamma)\frac{1}{K}\sum_{i=1}^K\sum_{j\ge 0}\mathds{1}[S_j=s] \gamma^{j}  \rho_{0:j-1} .
\end{align}

Note that
\begin{equation}
    \frac{n}{K}(1-\gamma)\mathbb{E}_{t \sim I_s} [\gamma^{\textrm{time}_t}  \rho_{\textrm{prod},t}]  = g_n(s)f_n(s).
    \label{eq:reformulate_RHS}
\end{equation}

\textbf{Prove the irreducibility.}
When studying states with non-negative discounted stationary distribution values under the target policy $\pi$, they can form an irreducible set with restarts.

$d_{\pi,\gamma}(s)>0$ implies that $d_{\mu}(s)>0$. Thus, $$\PP_{\mu}(S_j=s\text{, for some }j>0 \text{ and } j < T) > 0.$$
Meanwhile, we have $\PP_{\mu}(S_j=\varsigma\text{, for some }j>0|S_0=s)>0$ for episodic tasks. 
Thus, with restarts, given any two states $s$ and $s'$ with positive stationary distribution values,
\begin{align*}
    &\PP_{\mu}(S_j=s'|S_0=s)  \\
    & > \PP_{\mu}(S_j=\varsigma\text{, for some }j>0|S_0=s) \PP_{\mu}(S_j=s'\text{, for some }j>0 \text{ and } j < T) \\
    &>0.
\end{align*}

\textbf{Prove the positive recurrence.}
Note that a finite and irreducible Markov chain is positive recurrent.

\textbf{Prove the infinite number of trajectories.}
By Assumption \ref{assump:bound_traj_len}, the termination state is positive recurrent and is visited infinitely many times as the step $n$ goes to zero. Thus, there are infinitely many trajectories.

\textbf{Compute the almost sure limit of two functions.}
The function $g_n(s)$ is proven to converge to $g(s) = \EE_{\mu}[\tau_s^+(s)]$ by the ergodic theorem.

Apparently, $\lim_{n \to \infty} f_n(s) = \lim_{K \to \infty}(1-\gamma)\frac{1}{K}\sum_{i=1}^K\sum_{j\ge0}\mathds{1}[S_j=s] \gamma^{j}  \rho_{0:j-1} $.
By the central limit theorem, we have
\begin{equation}
    \lim_{n \to \infty} f_n(s) = \EE_{\mu} \left [(1-\gamma)\sum_{j\ge0}\mathds{1}[S_j=s] \gamma^{j}  \rho_{0:j-1} \right ] = d_{\pi,\gamma}(s).
\end{equation}

Thus,
\begin{align}
    L.H.S &= \lim_{n\to \infty} g_n(s)f_n(s)\\
    &= g(s) f(s) \\
    & = \EE_{\mu}[\tau_s^+(s)] d_{\pi,\gamma}(s)\\
    &= \cfrac{d_{\pi,\gamma}(s)}{d_{\mu}(s)}.
\end{align}
The last line follows Lemma \ref{lemma:form_stat_dist}, and the proof is completed.
\end{proof}

\subsection*{Proof of the Unbiasedness}

\begin{theorem}[Unbiasedness]
Given
\begin{itemize}
    \item a finite Markov decision process,
    \item a dataset $\cD$ collected under a behaviour policy $\mu$, and
    \item a target policy $\pi$
\end{itemize} 
such that Assumption \ref{assump:stationary_abso_cont} is satisfied, 
reweighting data by our average correction gives unbiased estimation for any function $f:\cS \to \RR$, that is,
\begin{equation}
\EE_{\cD} \left [\EE_{S \sim \cD}\left [c_{\cD}(S) f(S)\right ] \right ]  = \EE_{S \sim d_{\pi,\gamma}}[f(S)],
\end{equation}
where $\EE_{\cD}$ means expectation over trajectories sampled under the behaviour policy, and
$\EE_{S \sim \cD}$ representing sampling states uniformly from the dataset.
\end{theorem}
\vspace{-4mm}

\begin{proof}
    Denote the sampling distribution from the dataset as $\hat{d}(s) = \frac{\sum_{t=0}^{n-1}\mathds{1}[S_t=s]}{n}$.

    As proven in Corollary \ref{thm:density-ratio} in Equation \ref{eq:reformulate_RHS}, 
    \begin{equation}
        c_{\cD}(S) = \frac{1}{\hat{d}(s)} (1-\gamma)\frac{1}{K}\sum_{i=1}^K\sum_{j\ge 0}\mathds{1}[S^i_j=s] \gamma^{j}  \rho^i_{0:j-1}.
    \end{equation}

    Thus, $\EE_{S \sim \cD}\left [c_{\cD}(S) f(S)\right ] = \sum_{s\in\cS} f(s) (1-\gamma)\frac{1}{K}\sum_{i=1}^K\sum_{j\ge 0}\mathds{1}[S^i_j=s] \gamma^{j}  \rho^i_{0:j-1}$.

    After taking expectation over all $K$ trajectories, we have
    \begin{align}
        &\EE_{\cD} \left [ \sum_{s\in\cS} f(s) (1-\gamma)\frac{1}{K}\sum_{i=1}^K\sum_{j\ge 0}\mathds{1}[S^i_j=s] \gamma^{j}  \rho^i_{0:j-1} \right ]\\
        &= \sum_{s\in\cS} f(s) (1-\gamma) \sum_{j\ge 0} \gamma^j \PP_{\pi}(S_j=s)\\
        &= \sum_{s\in\cS} f(s)  d_{\pi,\gamma}(s).
    \end{align}
\end{proof}

\section*{Appendix B: Asymptotic Convergence}
We first introduce and prove the necessary lemmas. The proof of the convergence theorem is given in the second subsection.

\subsection*{Proof of the Required Lemma}
Denote the number of trajectories until step $t$ by $K(t)$. Recall our update rule is $d_{t+1} = d_t + \alpha_t (G_{t+1} d_t + g_{t+1})$ where $d_{t+1} = 
\begin{bmatrix}
    \theta_{t+1} \\
    \eta_{t+1}
\end{bmatrix}$, 
$G_{t+1} = 
\begin{bmatrix}
    -\phi(s_t) \phi(s_t)^\top - \lambda_1 I & -\lambda_2 H(1-\gamma) \phi(s_t)\\
    \lambda_2 H(1-\gamma) \phi(s_t) & -\lambda_2
\end{bmatrix}$, and 
$g_{t+1} = 
\begin{bmatrix}
    \phi(s_t) y_t \\
    -\lambda_2
\end{bmatrix}$.

\begin{lemma}
Define two matrices  $G = 
\begin{bmatrix}
    -\Phi^\top D_{\mu} \Phi - \lambda_1 I  & -\lambda_2 H (1-\gamma)\Phi^\top d_{\mu}\\
    \lambda_2 H (1-\gamma) d_{\mu}^\top \Phi & -\lambda_2
\end{bmatrix}$ and  $g = 
\begin{bmatrix}
    \frac{1}{(1-\gamma)\EE_{\mu}[T]} \Phi^\top D_{\mu} y\\
    -\lambda_2
\end{bmatrix}.$\\
When Assumption \ref{assump:bound_traj_len} and \ref{assump:convergence} are satisfied, we have
\begin{enumerate}
    \item 
        $\frac{1}{t+1} \sum_{k=0}^{t+1} G_k \to G \; \text{a.s. , and } \frac{1}{t+1} \sum_{k=0}^{t+1} g_k \to g \; \text{a.s. and in L$1$, as } t\to \infty.$
    \item The real parts of all eigenvalues of G are strictly negative.
\end{enumerate}
\label{thm:convergence_lemma}
\end{lemma}

\vspace{-4mm}

\begin{proof}
    \textbf{Let's prove the first point about almost sure convergence.} We will prove the convergence of each sub-matrix separately.

    Note that the ergodic theorem gives that the convergence of the top-left sub-matrix of $G_t$ as
    \begin{align*}
        &\frac{1}{t+1} \sum_{k=0}^{t+1} -\phi(s_k) \phi(s_k)^\top - \lambda_1 I \\
        &\xrightarrow{t \to \infty}  \sum_{s \in \cS} -d_{\mu}(s) \phi(s) \phi(s)^\top - \lambda_1 I \\
        &=-\Phi^\top D_{\mu} \Phi - \lambda_1 I.
    \end{align*}

    Similar convergence can be gained for the term $\lambda_2 H(1-\gamma) \phi(s_t)$ by the ergodic theorem as well.
    

    Combining these two results, we gain the almost sure convergence of $G_t$.

    Now we analyze the term $y_t\phi(s_t)=\sum_s \phi(s) \sum_{i \in I_s(t)} y_i$ where $I_s(t)$ denotes the showing up steps for a state $a$ until step $t$. 
    It can be further expressed as
    \begin{align*}
        &\frac{1}{t+1}\sum_{k=0}^t y_t\phi(s_t) \\
        &= \sum_s \frac{K(t)}{t+1} \phi(s) \frac{1}{1-\gamma}(1-\gamma) \frac{1}{K(t)} \sum_{i \in I_s(t)} y_i\\
        &\to \frac{1}{1-\gamma} \sum_s d_{\pi,\gamma}(s)\phi(s) \frac{1}{\EE_{\mu}[T]}\\
        &= \frac{1}{(1-\gamma)\EE_{\mu}[T]} \sum_s d_{\mu}(s) \frac{d_{\pi,\gamma}(s)}{d_\mu(s)}\phi(s).
    \end{align*}
    
    Note that the third line uses the convergence of $\frac{K(t)}{t+1}$ to $\frac{1}{\EE_{\mu}[T]}$.
    $K(t)$ goes to infinity as $t \to \infty$ since the recurrence time for the termination state has a finite expectation by Assumption \ref{assump:bound_traj_len}. Similarly, the ergodic theorem implies the almost sure convergence of $\frac{t+1}{K(t)}$ to $\EE_{\mu}[T]$.

    Furthermore, by central limit theorem, $(1-\gamma) \frac{1}{K(t)} \sum_{i \in I_s(t)} y_i$ converges to $d_{\pi,\gamma}(s)$. Since the limits of these two terms ar bounded, the limit of the product converges to the product of limits.

    For the L$1$-convergence, we analyze the expectation of $y_t\phi(s_t)$. Define $T_{K(t)}$ as the termination step of $K(t)$-th trajectory.
    \begin{align}
        &\frac{1}{t+1} \sum_{k=0}^{t+1} \EE_{\mu}[y_t\phi(s_t)] \\
        &= \frac{1}{t+1} \sum_{k=0}^{t+1} \gamma^t \PP_{\pi}(S_t=s) \phi(s)\\
        &=\sum_s \frac{K(t)}{t+1} \phi(s) \frac{1}{1-\gamma}(1-\gamma) \frac{1}{K(t)} [K(t) \sum_{j\ge 0} \gamma^j \PP_{\pi}(S_j=s) + \sum_{j=0}^{t- T_{K(t)}}\gamma^j \PP_{\pi}(S_j=s)]\\
        & \to \frac{1}{(1-\gamma)\EE_{\mu}[T]} \sum_s d_{\mu}(s) \frac{d_{\pi,\gamma}(s)}{d_\mu(s)}\phi(s).
    \end{align}

    Note that $\frac{1}{K(t)}\sum_{j=0}^{t- T_{K(t)}}\gamma^j \PP_{\pi}(S_j=s) \to 0$ as $K(t) \to \infty$ and $t \to \infty$.

    \textbf{Let's prove the second point about eigenvalues.}
    
    Let $\vartheta \in \mathbb{C}$, $\vartheta \neq 0$ 
    be a nonzero eigenvalue of $G$ with normalized eigenvector $x$, 
    that is $x^*x=1$, 
    where $x^*$ is the complex conjugate of $x$. 
    Hence, $x^*Gx=\vartheta$, $x\neq 0$. 
    Let $x^\top = (x_1^\top,x_2)$, 
    where $x_1 \in \mathbb{C}^d$ and $x_2 \in\mathbb{C}$. We can verify that 
    \begin{equation}
        \vartheta = -x_1^* (\Phi^\top D_\mu \Phi +\lambda_1 I) x_1+\lambda_2 x_2^*H(1-\gamma) d_\mu^\top \Phi x_1 - \lambda_2 x_1^*H(1-\gamma) d_\mu \Phi^\top x_2-\lambda_2 x_2^*x_2.
    \end{equation}

    Since $ d_\mu^\top \Phi$ is real, $\lambda_2 x_1^*H(1-\gamma) d_\mu \Phi^\top x_2 = (\lambda_2 x_2^*H(1-\gamma) d_\mu^\top \Phi x_1)^*$. It yields that the real part of their difference equals zero. Therefore, we have the real part of $\vartheta$, denoted by $\mathcal{R}(\vartheta)$, equals
    \begin{equation}
        \mathcal{R}(\vartheta) = -x_1^* (\Phi^\top D_\mu \Phi +\lambda_1 I) x_1-\lambda_2 x_2^*x_2.
    \end{equation}
    By the first point in Assumption \ref{assump:convergence}, we have $-x_1^* (\Phi^\top D_\mu \Phi +\lambda_1 I) x_1 \ge 0$, where the equality holds iff $x_1=0$. At least one of $\{x_1,x_2\}$ is nonzero. Consequently, we have $\mathcal{R}(\vartheta)<0$.
\end{proof}

\subsection*{Proof of the Convergence}

\begin{theorem}
Based on Assumption \ref{assump:bound_traj_len} and \ref{assump:convergence}, and Lemma \ref{thm:convergence_lemma}, we have
    \begin{equation}
        d_t \to -G^{-1} g \text{ a.s.}
    \end{equation}
which is the same fixed point for minimizing the following mean square error loss in Equation \ref{eq:mse_loss}.
\end{theorem}

The loss given a dataset $\cD$ to the true density ratio is written as 
\begin{align}
    \min_{\theta_{\text{mse}}}\mathcal{L}(\theta_{\text{mse}}; \cD) := \EE_{s_t \sim \cD} & \left [ \left (f_{\theta_{\text{mse}}}(s_t) - \frac{1}{(1-\gamma)\EE_{\mu}[T]} \frac{d_{\pi,\gamma}(s_t)}{d_{\mu}(s_t)} \right )^2 \right ] + \frac{\lambda_1 \lVert \theta_{\text{mse}} \rVert_2^2}{2} \nonumber\\
    &+ \lambda_2 \left (max_{\eta \in \RR} \EE_{s_t \sim \cD}[\eta  \frac{n}{K}(1-\gamma) f_{\theta_{\text{mse}}}(s_t) - \eta] -\frac{\eta^2}{2} \right ). 
    \label{eq:mse_loss}
\end{align}

\begin{proof}
First, we can verify some properties on our labels. Based on these properties and L$1$-convergence of $\frac{1}{t+1} \sum_{k=0}^tg_k$, we conclude that our label $(y_t,\phi(s_t) y_t)$ gives a unique invariant probability and is ergodic. 

The proof is the same as the corresponding proofs of (Yu, 2012, Theorem 3.2 and Prop. 3.2) for the case of off-policy LSTD.

\begin{enumerate}
    \item For any initial value of $\rho_{0:-1}$, $\sup_{t \ge 0 } \EE[\lvert (y_t,\phi(s_t) y_t) \rVert] < \infty$.
    \item Let $(y_t,\phi(s_t) y_t)$ and $(\hat{y_t},\phi(s_t) \hat{y_t})$ be defined by the same recursion and the same random variables, but with different initial conditions $\rho_{0:-1} \neq \hat{\rho}_{0:-1}$. Then, $y_t - \hat{y_t} \to 0 \textrm{ a.s.}$ and $\phi(s_t)y_t - \phi(s_t)\hat{y_t} \to 0 \textrm{ a.s.}$.
    \item $Z_t = (S_t, A_t, y_t, \phi(s_t) y_t)$ is a weak Feller Markov chain and bounded in probability.
\end{enumerate}
The proof follows ETD, since $y_t = \gamma^{\textrm{time}_t}\rho_{\textrm{prod},t}$ is a term in the ETD traces. For the second term, the difference between traces with different initializations for our correction and ETD is the same, so their proof also works here. The proof for the third claim follows the ETD paper.

Three conditions are required to use Theorem 6.1.1 in Kushner and Yin (2003) and follow the ETD proof (Theorem 4.1). Define $\xi_t = (y_t,S_t,A_t,S_{t+1})$ and $h(d,\xi_t) = G_t d + g_t$.

\begin{enumerate}
    \item \begin{equation}
        \frac{1}{t+1}\sum_{k=0}^t G_k \to G \textrm{ and } \frac{1}{t+1}\sum_{k=0}^t g_k \to g \textrm{ almost surely}.
    \end{equation}
    \item There exist nonnegative measurable functions $g_1(d)$, $g_2(\xi)$ such that $\lVert h(d,\xi) \rVert \le g_1(d)g_2(\xi)$ such that $g_1(d)$ is bounded on each bounded set, $\sum_{t \ge 0} \EE[g_2(\xi)] < \infty$, and $\frac{1}{t+1}\sum_{k=0}^t \left ( g_2(\xi_k) - \EE[g_2(\xi_k)]\right ) \to 0 \textrm{ almost surely}$. 
    \item There exist nonnegative measurable functions $g_3(d)$, $g_4(\xi)$ such that for each $d$ and $d'$, $\lVert h(d,\xi) -  h(d',\xi) \rVert \le g_3(d - d')g_4(\xi)$ such that $g_3(d)$ is bounded on each bounded set, $g_3(d) \to 0$ as $d \to 0$, $\sum_{t \ge 0} \EE[g_4(\xi)] < \infty$, and $\frac{1}{t+1}\sum_{k=0}^t \left ( g_4(\xi_k) - \EE[g_4(\xi_k)]\right ) \to 0 \textrm{ almost surely}$. 
\end{enumerate}

In our proof, the function $h(d,\xi)$ equals
\begin{equation}
    h(d,\xi) = \begin{bmatrix}
         -\phi(s) \phi(s)^\top - \lambda_1 I & -\lambda_2 H(1-\gamma) \phi(s_t)\\
    \lambda_2 H(1-\gamma) \phi(s_t) & -\lambda_2
    \end{bmatrix} d + \begin{bmatrix}
    \phi(s) y \\
    -\lambda_2
\end{bmatrix}.
\end{equation}

Then, for the second and third points, we first bound the norm of the matrix as followings.
\begin{align}
    &\lVert \begin{bmatrix}
         -\phi(s) \phi(s)^\top - \lambda_1 I & -\lambda_2 H(1-\gamma) \phi(s_t)\\
    \lambda_2 H(1-\gamma) \phi(s_t) & -\lambda_2
    \end{bmatrix}  \\
    &\le  \lVert \begin{bmatrix}
         -\phi(s) \phi(s)^\top - \lambda_1 I & 0\\
    0 & 0
    \end{bmatrix} \rVert + \lVert \begin{bmatrix}
         0 & -\lambda_2 H(1-\gamma) \phi(s_t)\\
    \lambda_2 H(1-\gamma) \phi(s_t) & -\lambda_2
    \end{bmatrix} \rVert 
   \\
    & \le
         \lVert -\phi(s) \phi(s)^\top - \lambda_1 I \rVert + \sqrt(2 \lVert \lambda_2 H(1-\gamma) \phi(s_t) \rVert^2)\lVert + \lambda_2
   \\
   & \le L^2+\lambda_1+\sqrt{2} \lambda_2 H(1-\gamma)L +\lambda_2 .
   \label{eq:matrix_norm}
\end{align}

\begin{align}
    \lVert h(d,\xi)\rVert & \le \lVert \begin{bmatrix}
         -\phi(s) \phi(s)^\top - \lambda_1 I & -\lambda_2 H(1-\gamma) \phi(s_t)\\
    \lambda_2 H(1-\gamma) \phi(s_t) & -\lambda_2
    \end{bmatrix} \rVert \lVert d \rVert + \lVert \begin{bmatrix}
    \phi(s) y \\
    -\lambda_2
    \end{bmatrix} \rVert \\
    & \le (L^2+\lambda_1+\sqrt{2} \lambda_2 H(1-\gamma)L +\lambda_2)\lVert d \rVert+ (\lVert \phi(s) y \rVert \lambda_2)\\
    & \le (L^2+\lambda_1+\sqrt{2} \lambda_2 H(1-\gamma)L +\lambda_2 + L y) (\lVert d \rVert+1).
\end{align}
Thus, $g_1(d) = (\lVert d \rVert+1)$ and $g_2(\xi)=L^2+\lambda_1+\sqrt{2} \lambda_2 H(1-\gamma)L +\lambda_2 + L y$.

We can bound the function norm using the matrix norm bound in Equation \ref{eq:matrix_norm}.
\begin{align}
    \lVert h(d,\xi) - h(d',\xi) \rVert & \le \lVert \begin{bmatrix}
         -\phi(s) \phi(s)^\top - \lambda_1 I & -\lambda_2 H(1-\gamma) \phi(s_t)\\
    \lambda_2 H(1-\gamma) \phi(s_t) & -\lambda_2
    \end{bmatrix} \rVert \lVert d-d' \rVert \\
    & \le (L^2+\lambda_1+\sqrt{2} \lambda_2 H(1-\gamma)L +\lambda_2 ) \lVert d-d' \rVert.
\end{align}
Thus, $g_3(d) = \lVert d \rVert$ and $g_4(\xi) = L^2+\lambda_1+\sqrt{2} \lambda_2 H(1-\gamma)L +\lambda_2 $ is a constant.

To show the fixed point is the same as minimizing the MSE to the true density ratio, we need to repeat the convergence proof for the new loss. But the only change is in the regression target and all other steps follow.
\end{proof}

\section*{Appendix C: Experimental Materials}

The hyperparameters of COP-TD are tuned the same as our method in the setting of dataset size $4000$, trajectory length $100$, discount factor $0.95$ and randomness coefficient $0.3$ for discrete-action tasks and $2.0$ for continuous-action tasks. Only one combination of the hyperparameters is used for all tasks.

For our algorithms, we test out the combination from parameter regularization coefficient $\lambda_1 \in [0, 0.001, 0.01, 0.1]$, distribution regularization parameter $\lambda_2 \in [0.5,2,10,20]$, and learning rate $\alpha \in [0.00005,0.0001,0.0005,0.001,0.005]$.

The neural network is set to be a two-hidden-layer neural network with hidden units $256$, which is the setting used by SR-DICE. The batch size is set the same as SR-DICE, equaling $512$.

The final choice of hyperparameters is shown in Table \ref{tab:hyperparameter}.

\begin{table}[h]
    \centering
    \begin{tabular}{c|c}
        \hline
         Parameter Regularizer $\lambda_1$ & 0.001 \\
         \hline
         Distribution Regularizer $\lambda_2$ & 0.5\\
         \hline
         Learning Rate & 0.0005 \\
         \hline
         Activation & ReLU \\
         \hline
    \end{tabular}
    \caption{}
    \label{tab:hyperparameter}
\end{table}

\subsection*{Ablation Study}
When training without the distribution regularization, the hyperparameters are also tuned with the regularization coefficient fixed, $\lambda_2=0$. Notice that the training step is much fewer than Figure \ref{fig:tune}. The results are plotted on a validation set. Without the distribution regularization, the ratio model is not learning except on CartPole.

\begin{figure}[t]
    \centering
    \vspace*{-6mm}
    \hspace*{-2mm}
    \includegraphics[width=1.1\textwidth
    ]{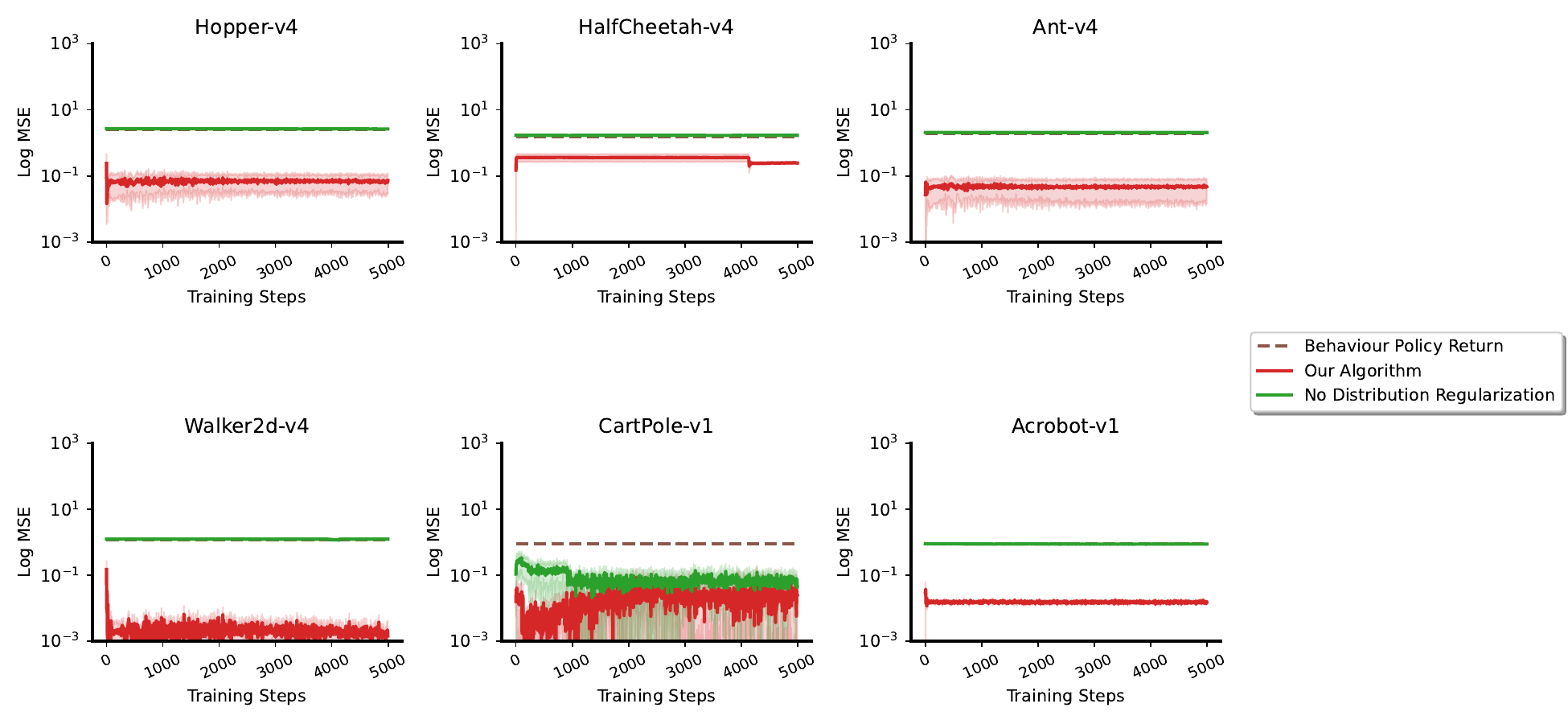}
    \vspace*{-2mm}
    \caption{This figure presents the mean square error of estimating the objective $J(\pi)$ in the log scale for each task. We present our method with and without the distribution regularization.}
    \label{fig:ablation}
    \vspace*{-6mm}
\end{figure}

\subsection*{Robustness}
The top-left subfigure of Figure \ref{fig:setting_walker} examines robustness against different discount factors. The main message is that our algorithm proves less robust to changes in the discount factor and would require re-tuning for each discount factor to achieve optimal performance. Our algorithm is robust to other changes except for a trajectory length that is too long. Results for other tasks are presented in Figure \ref{fig:setting_cartpole} and \ref{fig:setting_mujoco} and the conclusion holds for all tasks.

Turning to the bottom-right subfigure of Figure \ref{fig:setting_walker}, performance degrades when the trajectory length is set to $200$. A similar phenomenon is observed in Hopper and HalfCheetah. We propose two hypotheses for this drop. First, our method may struggle with very long trajectories, suggesting that users might benefit from truncating trajectories since longer horizons lead to heavily discounted and, thus, tiny labels. Second, the total dataset size is fixed at $4,000$, so increasing the length of each trajectory decreases the number of available trajectories. The approximation error decays sublinearly with the number of trajectories; thus, fewer trajectories can hinder performance.

\begin{figure}[ht]
    \centering
    \hspace*{-6mm}
    \begin{subfigure}[b]{0.45\textwidth}
         \centering
         \includegraphics[width=\textwidth]{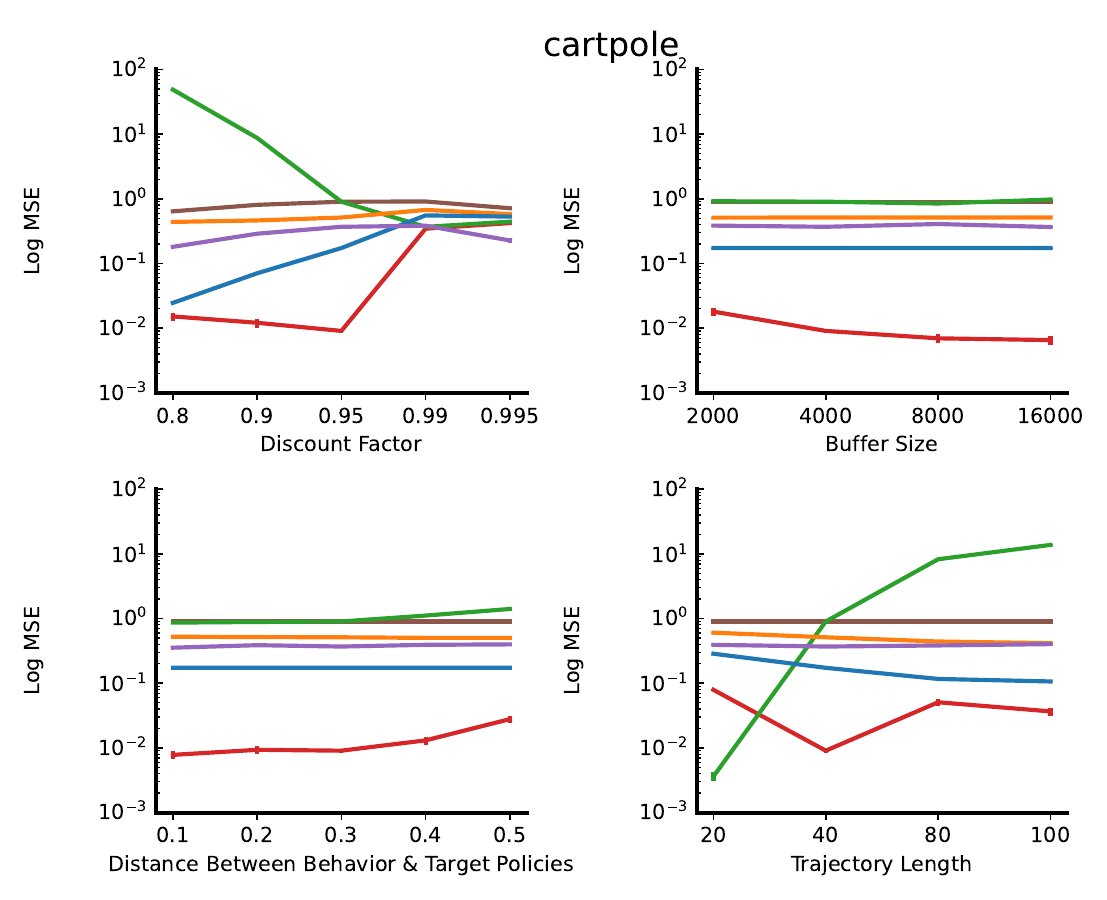}
         \caption{}
    \end{subfigure}
    \hfill
    \begin{subfigure}[b]{0.55\textwidth}
         \centering
         \includegraphics[width=\textwidth]{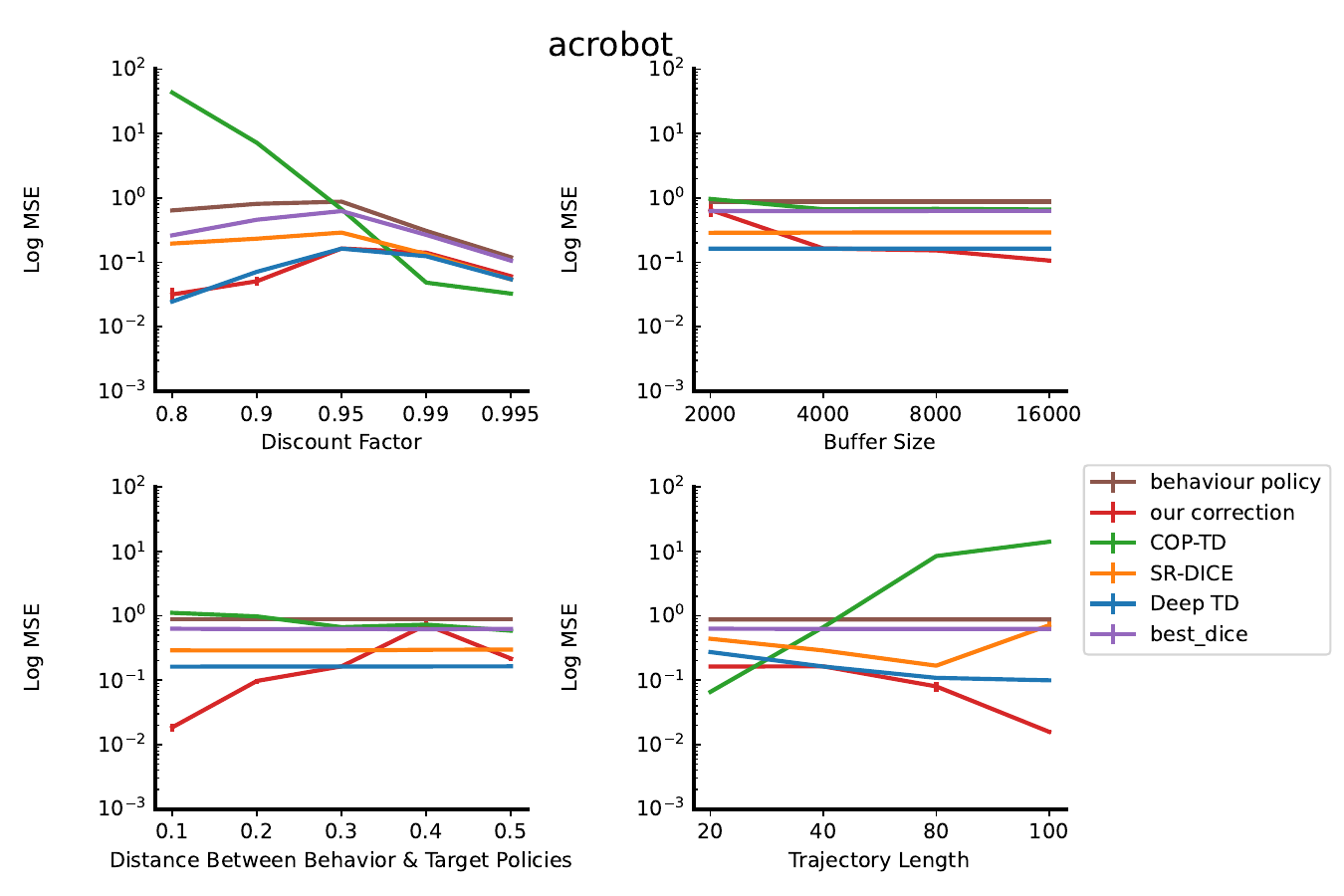}
         \caption{}
    \end{subfigure}
    \caption{This figure shows the robustness results on discrete-action tasks.}
    \label{fig:setting_cartpole}
    \vspace*{-4mm}
\end{figure}

\begin{figure}[ht]
    \centering
    \vspace*{-2mm}
    \hspace*{-2mm}
    \includegraphics[width=0.75\textwidth
    ]{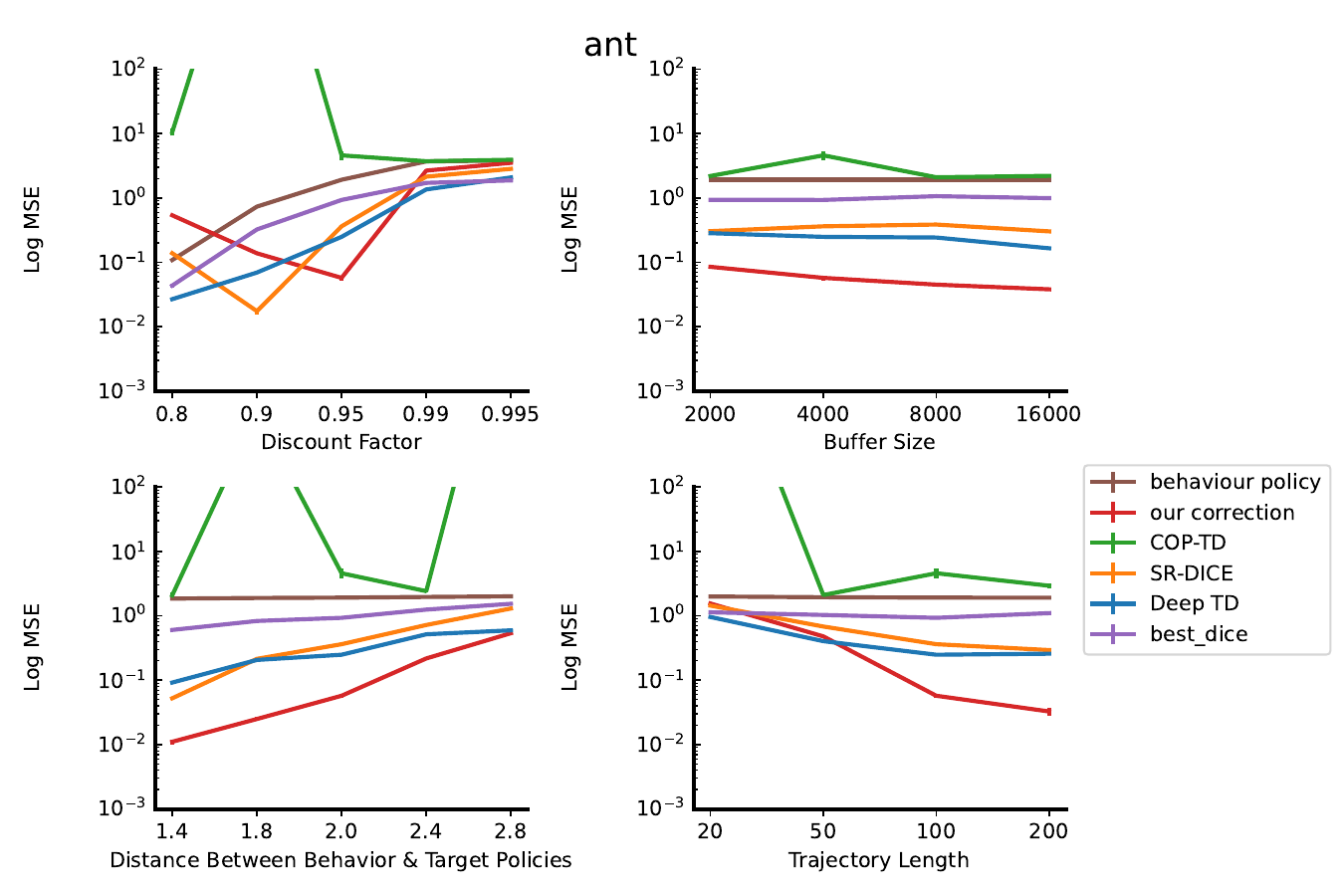}
    \includegraphics[width=0.75\textwidth
    ]{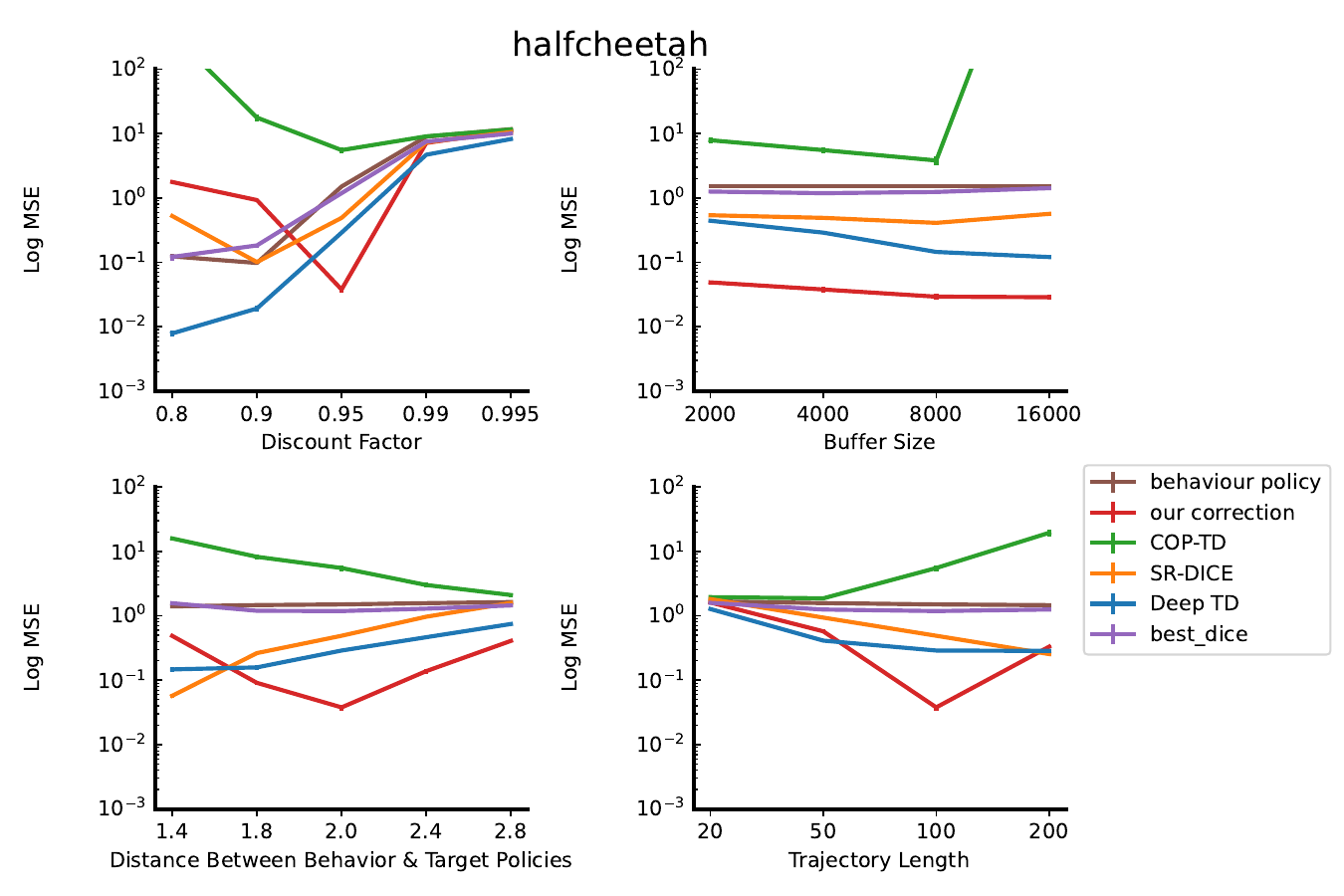}
    \includegraphics[width=0.75\textwidth
    ]{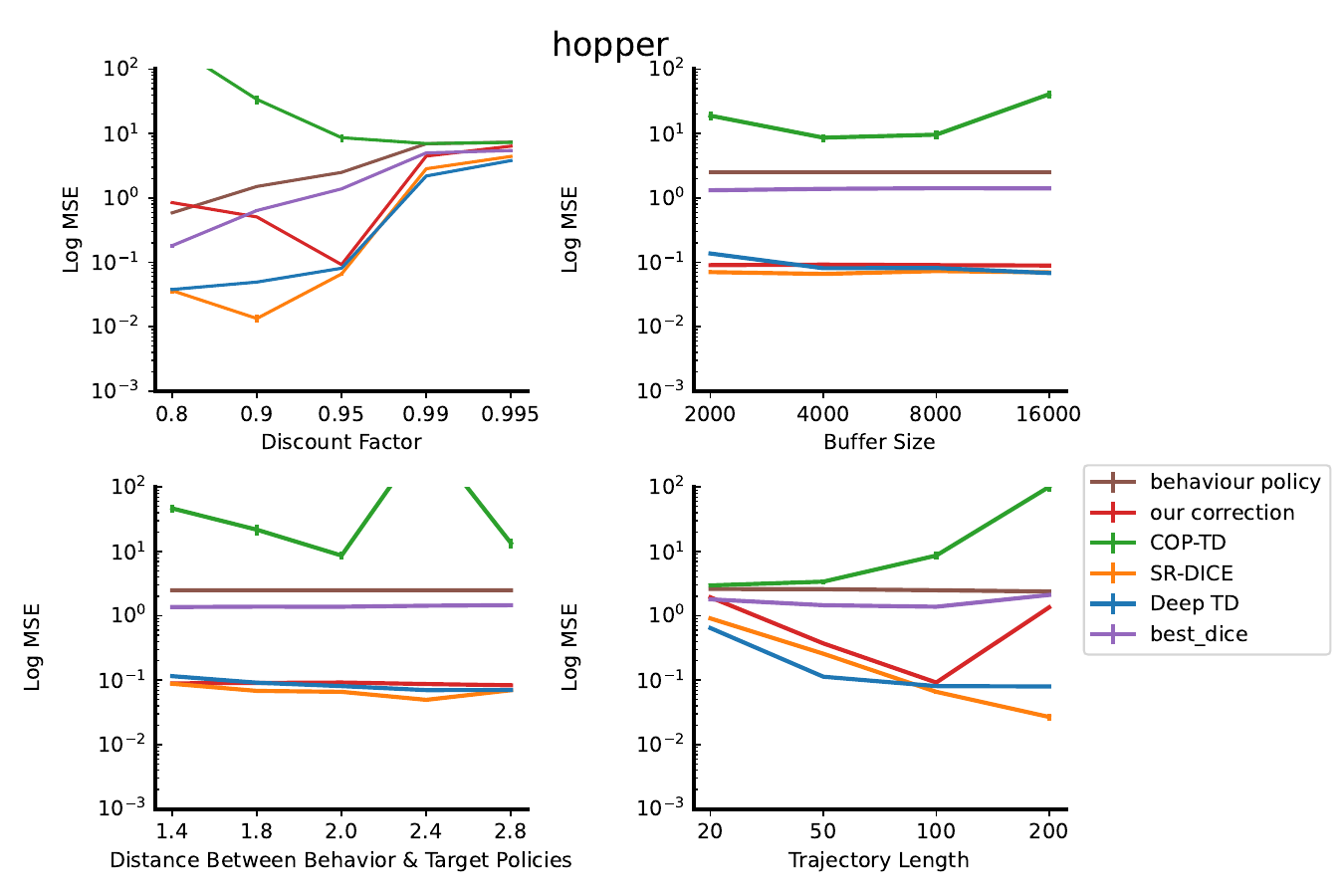}
    \caption{This figure shows the robustness results on continuous-action tasks.}
    \vspace*{-4mm}
    \label{fig:setting_mujoco}
\end{figure}

\end{document}